\documentclass[sigconf]{acmart}
\AtBeginDocument{%
  \providecommand\BibTeX{{%
    \normalfont B\kern-0.5em{\scshape i\kern-0.25em b}\kern-0.8em\TeX}}}

\newcommand{\ry}[1] {{\color{magenta} [RY: #1] } }

\newcommand{\out}[1]{}
\usepackage{algorithm}
\usepackage{algorithmic}

\usepackage{amsmath,amsfonts,bm}

\def\eqref#1{equation~\ref{#1}}

\def\1{\bm{1}}

\def\rd{{\textnormal{d}}}

\DeclareMathAlphabet{\mathsfit}{\encodingdefault}{\sfdefault}{m}{sl}
\SetMathAlphabet{\mathsfit}{bold}{\encodingdefault}{\sfdefault}{bx}{n}

\usepackage{url}

\usepackage{microtype}
\usepackage{graphicx}
\usepackage{booktabs} %
\usepackage{pifont}
\usepackage{hyperref}

\usepackage{caption}
\usepackage{natbib}

\usepackage{amsmath,amssymb,amsfonts,amsthm}

\usepackage{xcolor}
\usepackage{mathtools}
\usepackage{dsfont}
\usepackage{hyperref}
\usepackage{bm}
\usepackage{nicefrac}
\usepackage{wrapfig}
\usepackage{lipsum}

\usepackage{algorithm,algorithmic}
\usepackage[linesnumbered, ruled,vlined,algo2e]{algorithm2e}

\newcounter{assumption}%
\renewcommand{\theassumption}{\arabic{assumption}}

\def\rd{\mathrm{d}}
\def\rT{{\rm T}}

\def\lX{\mathcal{X}}
\def\lY{\mathcal{Y}}

\def\normal{\mathcal{N}}

\newcommand*\E[1]{\mathbb{E}\left[#1\right]}
\newcommand*\Ep[2]{\mathbb{E}_{#1}\left[#2\right]}

\newcommand*\lrbb[1]{\left\{#1\right\}}
\newcommand*\lrp[1]{\left(#1\right)}

\newcommand*\ind[1]{{\mathds{1}\lrbb{#1}}}

\def\rd{\mathrm{d}}
\def\rT{{\rm T}}

\newcommand{\real}{\ensuremath{\mathbb{R}}}

\newcommand{\MIS}{\mathrm{MIS}}

\newcommand{\NORMAL}{\ensuremath{\mathcal{N}}}

\newtheorem{proposition}{Proposition}

\usepackage{multirow}
\usepackage{caption}
\usepackage{subcaption}

\newcommand{\ts}{\textsuperscript}
\DeclareMathOperator*{\argmin}{argmin}

\newcommand{\V}[1]{{\mathbf{#1}}} 

\newcommand{\xmark}{$\times$}%

\usepackage{booktabs}
\usepackage{amssymb,amsmath} 

\renewcommand{\edits}[1]{{\color{black}#1}}

\setcopyright{rightsretained} 

\begin{document}
\settopmatter{authorsperrow=4}
\title{Quantifying Uncertainty in Deep Spatiotemporal Forecasting}

\author{Dongxia Wu}
\affiliation{%
  \institution{UC San Diego}
  \city{La Jolla}
  \state{CA}
  \country{USA}
}
\email{dowu@ucsd.edu}

\author{Liyao Gao}
\affiliation{%
  \institution{University of Washington}
  \city{Seattle}
  \state{WA}
  \country{USA}
  }
\email{marsgao@uw.edu}

\author{Xinyue Xiong}
\affiliation{%
 \institution{Northeastern University}
 \city{Boston}
 \state{MA}
 \country{USA}
}
\email{xiong.xin@northeastern.edu}

\author{Matteo Chinazzi}
\affiliation{%
 \institution{Northeastern University}
 \city{Boston}
 \state{MA}
 \country{USA}
 }
 \email{m.chinazzi@northeastern.edu}

\author{Alessandro Vespignani}
\affiliation{%
  \institution{Northeastern University}
  \city{Boston}
  \state{MA}
  \country{USA}
 }
 \email{a.vespignani@northeastern.edu}

\author{Yi-An Ma}
\affiliation{
  \institution{UC San Diego}
  \city{La Jolla}
  \state{CA}
  \country{USA}
}
\email{yianma@ucsd.edu}

\author{Rose Yu}
\affiliation{
  \institution{UC San Diego}
  \city{La Jolla}
  \state{CA}
  \country{USA}
}
\email{roseyu@ucsd.edu}
\renewcommand{\shortauthors}{Wu, et al.}

\begin{abstract}
 Deep learning is gaining increasing popularity for spatiotemporal forecasting. However,  prior works have mostly focused on point estimates without quantifying the uncertainty of the predictions. In high stakes domains, being able to generate probabilistic forecasts with confidence intervals is critical to risk assessment and decision making. Hence, a systematic study of uncertainty quantification (UQ) methods for spatiotemporal forecasting is missing in the community. In this paper, we describe two types of spatiotemporal forecasting problems: regular grid-based and graph-based.  Then we analyze  UQ methods from both the Bayesian and the frequentist point of view, casting in a unified framework via statistical decision theory. Through extensive experiments on real-world road network traffic, epidemics, and air quality forecasting tasks, we reveal the statistical and computational trade-offs for different UQ methods: Bayesian methods are typically more robust in mean prediction, while confidence levels obtained from frequentist methods provide more extensive coverage over data variations. Computationally, quantile regression type methods are cheaper for a single confidence interval but require re-training for different intervals. Sampling based methods generate samples that can form multiple confidence intervals, albeit at a higher computational cost. 
\end{abstract}



\keywords{uncertainty quantification, deep neural networks, spatiotemporal forecasting}


\copyrightyear{2021}
\acmYear{2021}
\acmConference[KDD '21]{Proceedings of the 27th ACM SIGKDD Conference on Knowledge Discovery and Data Mining}{August 14--18, 2021}{Virtual Event, Singapore}
\acmBooktitle{Proceedings of the 27th ACM SIGKDD Conference on Knowledge Discovery and Data Mining (KDD '21), August 14--18, 2021, Virtual Event, Singapore}\acmDOI{10.1145/3447548.3467325}
\acmISBN{978-1-4503-8332-5/21/08}
\fancyhead{}
\maketitle

\section{Introduction}

Forecasting the spatiotemporal evolution of time series data is at the core of data science. Spatiotemporal forecasting exploits the spatial correlations to improve the forecasting performance for multivariate time series \cite{cressie2015statistics}. Many works have used deep learning for spatiotemporal forecasting, applied to various domains from weather \cite{xingjian2015convolutional,yi2018deep} to traffic \cite{li2018diffusion,geng2019spatiotemporal} forecasting. However, the majority of existing works focus on point estimates without generating confidence intervals. The resulting predictions fail to provide credibility for assessing potential risks and assisting decision making, which is critical especially in high-stake domains such as public health and transportation.

It is known that deep neural networks are poor at providing uncertainty and tend to produce overconfident predictions \citep{amodei2016concrete}.  Therefore, uncertainty quantification (UQ) is receiving growing interests in deep learning~\citep{gal2016dropout, vandal2018quantifying, qiu2019quantifying, wen2017multi}. 
There are roughly two types of UQ methods in deep learning. One leverages frequentist thinking and focuses on robustness. Perturbations are made to the inference procedure in initialization ~\citep{lakshminarayanan2017simple,fort2019ensemble},  neural network weights~\citep{gal2016dropout}, and datasets~\citep{lee2015m,osband2016DQN}. The other type is Bayesian, which aims to model posterior beliefs of network parameters given the data \citep{neal2012bayesian,kingma2013auto,heek2019}. 
A common focal point of both approaches is the generalization ability of the neural networks. In forecasting, we not only wish to explain in-distribution examples, but also generalize to future data with model misspecifications and distribution shifts. Time series data provide the natural laboratory for studying this problem.




Spatiotemporal sequence data pose unique challenges to UQ: (1) Spatiotemporal dependency: accurate forecasts require models that can capture both spatial and temporal dependencies.   (2)  Evaluation metrics: point estimates often use mean squared error (MSE) as a metric, which  only measures the averaged point-wise error for the sequence.   For uncertainty, common metrics such as held-out log-likelihood require exact likelihood computation. The number of possibilities grows  exponentially with the length of the sequence.  (3) Forecasting horizon: the sequential dependency in time series often leads to error propagation for long-term forecasting, so does uncertainty. Unfortunately, most existing deep UQ methods deal with scalar classification/regression problems   with a few exceptions. \cite{vandal2018quantifying} proposes to use concrete dropout for climate down-scaling but not for forecasting.  \cite{wang2019deep} proposes a Bayesian neural network based on variational inference. Still, a systematic study of UQ methods for spatiotemporal forecasting with deep learning is missing in the data science community.


In this paper, we conduct a  benchmark study of deep uncertainty quantification for spatiotemporal forecasting. We investigate spatiotemporal forecasting problems on a regular grid as well as on a graph. We compare 2 types of UQ methods: frequentist methods including Bootstrap, Quantile regression, Spline Quantile regression and Mean Interval Score regression;  and Bayesian methods including Monte Carlo dropout, and Stochastic Gradient Markov chain Monte Carlo (SG-MCMC).  We analyze the properties of both frequentist and Bayesian  UQ methods, as well as their practical performances. We further provide a recipe for practitioners when tackling UQ problems in deep spatiotemporal forecasting.  We experimented with many real-world forecasting tasks to validate our hypothesis: road network traffic forecasting, epidemic forecasting, and air quality prediction.

Our contributions include:
\begin{itemize}
    \item  We conduct the first systematic benchmark study for deep learning uncertainty quantification (UQ) in the context of multi-step spatiotemporal forecasting.
    \item  We cast frequentist and Bayesian UQ  methods under a unified framework and adapt 6   UQ methods for grid-based and graph-based spatiotemporal forecasting problems.
    \item Our study reveals that posterior sampling excels at mean prediction whereas quantile and mean interval score regression obtain confidence levels that cover data variations better.
    \item We also notice that the sample complexity of posterior sampling is lower than that of the bootstrapping method, potentially due to posterior contraction~\citep{wilson2020bayesian}.
\end{itemize}

\section{Related Work}



\subsection{Spatiotemporal Forecasting} 
\out{
Classic time series models such as ARMA or  ARIMA were developed for univariate time series~(see \citep{hyndman2018forecasting} and the references therein. For multivariate time series, \citep{yu2016temporal} exploit information sharing across variables by applying matrix factorization. \citep{salinas2020deepar, wang2019deepfactor} introduce latent variables and use RNNs to approximate the parameters in the likelihood functions. These probabilistic forecasting methods provide uncertainty directly, but assumes factorizability over time or exchangeability over space on the model likelihood function, neither holds true in our current setting. \citep{gasthaus2019probabilistic} proposes spline quantile function.  
}
Spatiotemporal forecasting is challenging due to the higher-order dependencies in space, time, and variables. Classic Bayesian statistical models \cite{cressie2015statistics} often require strong modeling assumptions. There is a growing literature on using deep learning due to its flexibility. Depending on the geometry of the space, the existing methods fall into two categories: regular grid-based such as Convolutional LSTM \citep{xingjian2015convolutional,lin2020preserving, wang2017predrnn,yao2018deep,yao2019revisiting} and graph-based such as  Graph Convolutional LSTM \citep{yu2018spatio, geng2019spatiotemporal,li2018diffusion}. The high-level idea behind these architectures is to integrate spatial embedding using (Graph) Convolutional neural networks with deep sequence models such as LSTM. However, most existing works in deep learning forecasting only produce point estimation rather than probabilistic forecasts with built-in uncertainty. 


\subsection{Uncertainty Quantification}
Two types of deep UQ methods prevail in data science. Bayesian methods model the posterior beliefs of  network parameters given the data. For example, 
 {Bayesian neural networks learning (BNN) uses approximate Bayesian inference to improve inference efficiency.  \citet{wang2016natural} proposes natural parameter network using  exponential family distributions. \citet{shekhovtsov2018feed} proposes new approximations for categorical transformations. However, these BNNs are focused on classification problems rather than forecasting.}   \cite{vandal2018quantifying} use concrete dropout for climate downscaling. \citep{wang2019deep} proposes a Bayesian deep learning model  VI for weather forecasting. On the other hand, frequentist UQ methods emphasize the robustness against variations in the data. For instance, \cite{pearce2018high,gasthaus2019probabilistic} propose prediction intervals as the forecasting objective. \citep{alaa2020frequentist} proposed a UQ method with bootstrapping using the influence function. This paper extends previous works and studies the efficacy and efficiency of various methods for UQ in spatiotemporal forecasts. We defer the detailed discussion of these methods to Section \ref{sec:bayesian_uq}. To the best of our knowledge, there does not exist a systematic benchmark study of these UQ methods for deep spatiotemporal forecasting.

\section{Deep Spatiotemporal Forecasting}
\subsection{Problem Definition}
\label{sec:deep_models}
Given multivariate time series $\mathcal{X} = (\V{X}_1, \cdots, \V{X}_t) $ of $t$ time step, where $\V{X}_t \in \mathbb{R}^{P\times D}$ indicating $D$ features from $P$ locations. We also have a spatial  correlation matrix ${A}\in \mathbb{R}^{P\times P}$ indicating the spatial proximity.  Deep spatiotemporal forecasting approximates a function $f$ with deep neural networks such that:
\begin{equation}
    f: (\mathcal{X}; {A}) \rightarrow (\V{X}_{t+1}, \cdots, \V{X}_{t+H}; {A})
\label{eqn:spatiotemporal_forecast}
\end{equation}
where $H$ is the forecasting horizon. 

Spatiotemporal forecasting exploits  spatial correlations to improve the prediction performance.
%
A popular idea behind existing deep learning models is 
to integrate a spatial embedding of  $\V{X}_t$ into a deep sequence model such as an RNN or LSTM. Mathematically speaking, that means to replace the matrix multiplication operator in an RNN with a convolution or graph convolution operator. 

For grid-based data we can use convolution, which is the key idea behind ConvLSTM \cite{xingjian2015convolutional}. As the locations are distributed as a grid, we reshape $\V{X}_t$ into a tensor of shape $W\times H \times D$ with $W\times H =P$ representing the width and the height of the grid. A convolutional RNN model becomes:
\begin{equation}
    \V{h}_{t+1}= \sigma(W^h\cdot \V{h}_{t} + W^x\cdot \V{X}_{t}) \longrightarrow   \V{h}_{t+1}= \sigma(W^h \ast \V{h}_{t} + W^x\ast \V{X}_{t})
\label{eqn:convlstm}
\end{equation}
where $\V{h}_t $ are the hidden states of the RNN and $\ast$ is the convolution $(W \ast \V{X})[m,n]=  \sum_i\sum_{j}W[i,j]\V{X}_t[m-i,n-j]$.

For graph-based data we use graph convolution, leading to graph convolutional RNN, similar to the design of DCRNN \cite{li2018diffusion}:
\begin{equation}
   \V{h}_{t+1}= \sigma(W \ast_g \V{h}_{t} + W\ast_g \V{X}_{t})
\label{eqn:graph_convlstm}
\end{equation}
Where $\ast_g$ stands for graph convolution. There are many variations for graph convolution or graph embedding. One example is  
\begin{equation}
W \ast_g \V{X}_t = W\cdot (D^{-1}A)\cdot \V{X}_t
\label{eqn:convolution}
\end{equation}
Here $D \in \mathbb{R}^{P\times P}$ contains the diagonal element of $A$.  One can also replace the random walk matrix $D^{-1}A$ with the normalized Laplacian matrix $I -D^{1/2}(D-A)D^{1/2}$ as in \cite{kipf2016semi}.

These deep learning models are developed only for point estimation. We will use them as base models to design and compare various deep UQ methods for probabilistic forecasts.
\subsection{Evaluation Metrics}
\label{sec:MIS}
For point estimation, mean squared error (MSE) or mean absolute error (MAE) are well-established metrics. But for probabilistic forecasts, MSE is insufficient as it cannot characterize the quality of the prediction confidence. While held-out likelihood is the gold standard  in the  existing UQ literature, it is difficult to use for deep learning as they often do not generate explicit likelihood outputs \citep{kingma2018glow}. Instead, we revisit a key concept from statistics and econometric called Mean Interval Score \citep{gneiting2007strictly}. 


Mean Interval Score (MIS)  is a scoring function for interval forecasts. \edits{It is also known as Winkler loss \citep{askanazi2018comparison}.} It rewards narrower confidence or credible intervals and encourages  intervals that  include the observations (coverage). From a computational perspective,
MIS is preferred  over other scoring functions such as the Brier score \citep{brier1950verification}, Continuous Ranked Probability Score (CRPS) \citep{matheson1976scoring, hamill1995probabilistic, gneiting2007strictly} as it is intuitive and  easy to compute. MIS only requires three quantiles while CRPS requires integration over all quantiles. 
From a statistical perspective, MIS does not constrain the model to be parametric, nor does it require explicit likelihood functions. It is therefore better suited for comparison across a wide range of methods.

We formally define MIS for estimated upper and lower confidence bounds.
For  a one-dimensional random variable $Z\sim\mathbb{P}_Z$, if the estimated upper and lower confidence bounds are $u$ and $l$, \edits{where $u$ and $l$ are the $(1-\frac{\rho}{2})$ and $\frac{\rho}{2}$ quantiles for the $(1-\rho)$ confidence interval}, MIS is defined using samples $z_i\sim\mathbb{P}_Z$:
\begin{align*}
    \MIS_N(u,l;\rho) = \frac{1}{N} \sum_{i=1}^N \Big\{ \lrp{u-l} &+ \frac{2}{\rho}(z_i-u)\ind{z_i>u} \\
    &+\frac{2}{\rho}(l-z_i)\ind{z_i<l}\Big\}.
\end{align*}
In the large sample limit, MIS converges to the following expectation:
\[
\textstyle \MIS_{\infty}(u,l;\rho) = (u-l) + \frac{2}{\rho} \lrp{ \E{Z-u|Z>u} + \E{l-Z|Z<l} }.
\]
We prove in the following two novel propositions about the consistency of MIS as a scoring function, which demonstrates that minimizing MIS will lead to unbiased estimates for the confidence intervals.
The first proposition is concerned with the consistency of MIS in the large sample limit and the second studies the finite sample consistency of it.
Proofs for both propositions are deferred to Appendix~\ref{Appendix:theory}.
\begin{proposition}
Assume that the distribution $\mathbb{P}_Z$ of $Z$ has a probability density function.
Then for $(u^*,l^*) = \argmin_{u>l; u, l \in \real} \MIS_{\infty} (u,l;\rho)$, $[l^*,u^*]$ is the $(1-\rho)$ confidence level.
\label{thm:mis}
\end{proposition}

\begin{proposition}
\label{thm:mis_finite}
For $(u^*,l^*) = \arg\min_{u>l; u, l \in \real} \MIS_{N} (u,l;\rho)$, $[l^*,u^*]$ is the $(1-\rho)$ quantile of the empirical distribution formed by the samples $\{z_1,\cdots,z_N\}$.
\end{proposition}

Note from Proposition~\ref{thm:mis_finite} that the optimal interval for $\MIS_{N}$ contains $\rho$ portion of the data points, because of the balance it strikes between the reward towards narrower intervals and the penalty towards excluding observations. Next, we describe a unified framework based on statistical decision theory to better understand uncertainty quantification.

\section{Uncertainty Quantification in Spatiotemporal Forecasts}
\label{sec:uq_framework}
We  cast uncertainty quantification  from frequentist and Bayesian perspectives into a single framework.
Then we describe both the Bayesian and frequentist techniques to enable  deep spatiotemporal forecasting models to generate probabilistic forecasts.

\subsection{A Unified Framework}
We start from a statistical decision theory point of view to examine uncertainty quantification.
Assume that each datum $\lX$ and its label $\lY$ are generated from a mechanism specified by a probability distribution with parameter $\theta$: $p(\lY|\lX;\theta)$.
We further assume that the parameter $\theta$ is distributed according to $p(\theta)$.
The goal of a probabilistic inference procedure is to minimize the statistical risk, which is an expectation of the loss function over the distribution of the data as well as $p(\theta)$, the distribution of the class of models considered.
Concretely, consider an estimator $\hat{\theta}$ of $\theta$, which is a measurable function of the dataset $\{\lX\} = \{\lX_0,\cdots,\lX_N\}$.
We wish to minimize its risk:
\begin{eqnarray} 
\textstyle {\rm Risk}(\hat{\theta},\theta) &=& \E{\|f(\hat{\theta}) - f(\theta)\|^2}\nonumber \\
&=& \Ep{\theta}{ \E{\|f(\hat{\theta}) - f(\theta)\|^2 \Big| \theta} } \label{eq:frequentist}  \\
&=&  \E{ \Ep{\theta}{\|f(\hat{\theta}) - f(\theta)\|^2 \Big| \{\lX\}} } \label{eq:Bayesian},
\end{eqnarray}
where in this paper, function $f$ is taken to be  the prediction of the deep learning model given its parameters.

When we design an inferential procedure, we can use~\eqref{eq:frequentist} or~\eqref{eq:Bayesian}.
If we take the procedure of~\eqref{eq:frequentist} and be a frequentist, we first seek the estimator $\hat{\theta}$ to minimize 
$\mathbb{E}[\|f(\hat{\theta}) - f(\theta)\|^2 | \theta]$.
In practice, we minimize the empirical risk over the training dataset $\{\lX\}$, $\frac{1}{N}\sum_{i\in\{0,\cdots,N\}} (f(\lX_i;\hat{\theta}) - \lY_i)^2$, instead of the original expectation that is agnostic to the algorithms.

If we take the procedure of~\eqref{eq:Bayesian} and be Bayesian, we first take the expectation over the distribution of the parameter space $\theta$, conditioning on the training dataset $\{\lX\}$.
We can  find that minimizing $\Ep{\theta}{\| f(\hat{\theta})-f(\theta) \|^2 \Big| \{\lX\}}$ is equivalent to taking $f(\hat{\theta})$ to be $\Ep{\theta}{f(\theta)}$.
For a proper prior distribution over $\theta$, this boils down to the posterior mean of $f(\theta)$.
In this work, we also take function $f$ to be the output of the model.

To quantify uncertainty of the respective approaches, it is important to capture variations in $\theta$ that determine the distribution of the data.
There are two sources of in-distribution uncertainties corresponding to different modes of variations in $\theta$: effective dimension of the estimator $\hat{\theta}$ being smaller than that of the true $\theta$, oftentimes addressed to as the bias correction problem; and the variance of the estimator $\hat{\theta}$, often approached via parameter calibration.
On top of that, we are also faced with distribution shift that poses additional challenges.

Frequentist methods aim to directly capture variations in the data. For example, quantile regression explores the space of predictions within the model class by ensuring proper coverage of the data set. bootstrap method resamples the data set to target both the bias correction and parameter calibration problems, albeit possessing a potentially high resampling variance.
Bayesian methods often leverage prior knowledge about the distribution of $\theta$ and use MCMC sampling to integrate different models to capture both variations in the parameter space.

In practice, it is important to understand the sample complexity (e.g., number of resampling runs in bootstrap or number of MCMC samples) of all methods and how it scales with the model and data complexity.
For deep learning models, it is also crucial to capture the trade-off between the variance incurred by the algorithms and their computation complexity.
To understand the performance of the uncertainty quantification methods, we apply the statistically consistent metric MIS to measure how well the confidence intervals computed from either frequentist or Bayesian methods capture the uncertainty of the predictions given the data.
In what follows, we introduce the computation methods for uncertainty quantification.
\subsection{Frequentist UQ Methods}
We start by describing the frequentist methods.
The general idea behind them is to directly capture variations in the dataset, either by resampling the data or by learning an interval bound to encompass the dataset.

\textbf{Bootstrap.}
The (generalized) bootstrap method~\citep{efron2016} randomly generates in each round a weight vector over the index set of the data.
The data are then resampled according to the weight vector.
With every resampled dataset, we retrain our model and generate predictions.
Using the predictions from different retrained models, we infer the $(1-\rho)$ confidence interval from the $(1-\frac{\rho}{2})$ and $\frac{\rho}{2}$ quantiles based on Proposition~\ref{thm:mis_finite}. The quantiles can be estimated using the order statistics of Monte Carlo samples.  

\textbf{MIS Regression}
For a fixed confidence level $\rho$, we can directly minimize MIS  to obtain estimates of the confidence intervals.
\edits{
Specifically, we use MIS as a loss function for deep neural networks, we use a multi-headed model to jointly output the upper bound $u(x)$, lower bound $l(x)$, and the  prediction $f(x)$ for a given input $x$, and minimize the neural network parameter $\theta$:
\begin{align}
\lefteqn{L_{\text{MIS}}(y,u(x),l(x),f(x);\theta, \rho)} \nonumber\\
&= \min\limits_{\theta}\Big\{ \textstyle  \mathbb{E}_{(x,y)\sim \mathcal{D}}[ (u(x)-l(x))+ \frac{2}{\rho}(y-u(x))\ind{y>u(x)} \nonumber\\ 
&\;\quad\qquad \textstyle + \frac{2}{\rho}(l(x)-y)\ind{y<l(x)} + |y-f(x)|]   \Big\} \label{eqn:MIS_loss}
\end{align}

Here $\ind{\cdot}$ is an indicator function, which can be implemented using the identity operator over the larger element in Pytorch.
}

\textbf{Quantile Regression}
We can use the one-sided quantile loss function ~\citep{Koenker2005} to generate predictions for a fixed confidence level  $\rho$. \edits{Given an input $x$, and the output $f(x)$ of a neural network, parameterized by $\theta$, quantile loss is defined as follows:
\begin{align}
\lefteqn{ L_{\text{Quantile}}(y, f(x); \theta, \rho) } \nonumber\\ 
&= \min\limits_{\theta}\Big \{ \mathbb{E}_{(x,y)\sim \mathcal{D}}[(y-f(x))(\rho-\ind{y<f(x)}]\Big \} 
 \label{eqn:pinball_loss}
\end{align}
}
Quantile regression  behaves similarly as the MIS regression method. Both methods generate one confidence interval per time. \edits{Similar ideas of directly optimizing the prediction interval using different variations of quantile loss have also been explored  by \citet{kivaranovic2020adaptive,  tagasovska2019single,pearce2018high}.}

One caveat of these methods is that different predicted quantiles can cross each other due to variations given finite data.
This will lead to a strange phenomenon when the size of the data set and the model capacity is limited: the  higher confidence interval does not contain the interval of lower confidence level or even the point estimate.
One remedy for this issue is to add variations in both the data and the parameters to increase the effective data size and model capacity.
In particular, during training, we can use different subsets of data and repeat the random initialization from a prior distribution to form an ensemble of models.
In this way, our modified MIS and quantile regression methods have integrated across different models to quantify the prediction uncertainty and have taken advantage of the Bayesian philosophy.


Another solution to alleviate the quantile crossing issue is to minimize CRPS \edits{by assuming the quantile function to be a piecewise linear spline with
monotonicity}~\citep{gasthaus2019probabilistic}, a method we call Spline Quantile regression (SQ).
In the experiment, we also included this method for comparison.

\subsection{Bayesian UQ Methods}
\label{sec:bayesian_uq}
We describe the Bayesian approaches to UQ.
These methods impose a prior distribution over the class of models and average across them, either via MCMC sampling or variational approximations.

\textbf{Stochastic Gradient MCMC (SG-MCMC)}
To estimate expectations or quantiles according to the posterior distribution over the parameter space, we use SG-MCMC~\citep{welling2011bayesian,ma2015complete}.
We find in practice that the stochastic gradient thermostat method (SGNHT)~\citep{ding2014bayesian,shang2015} is particularly useful in controlling the stochastic gradient noise.
This is consistent with the observation in~\citep{heek2019} where stochastic gradient thermostat method is applied to an i.i.d. image classification task.

To generate samples of model parameters $\theta$ (with a slight abuse of notation) according to SGNHT, we first denote the loss function (or the negative log-likelihood) over a minibatch of data as $\widetilde{\mathcal{L}}(\theta)$.
We then introduce hyper-parameters including the diffusion coefficients $A$ and the learning rate $h$ and make use of auxiliary variables $p\in\real^{d}$ and $\xi\in\real$ in the algorithm.
We randomly initialize $\theta$, $p$, and $\xi$ and update according to the following update rule.
\begin{align}
\label{eq:SGNHT}
\left\{
\begin{array}{l}
    \theta_{k+1} = \theta_k + p_k h  \\
    p_{k+1} = p_k - \nabla \widetilde{\mathcal{L}}(\theta) h - \xi_k p_k h + \mathcal{N}(0,2A h) \\
    \xi_{k+1} = \xi_k + \lrp{ \frac{p_k^\rT p_k}{d} - 1 } h.
\end{array}
\right.
\end{align}
Upon convergence of the above algorithm at $K$-th step, $\theta_K$ follows the distribution of the posterior.
We run parallel chains to generate different samples according to the posterior and quantify the prediction uncertainty.

\textbf{Approximate Bayesian Inference}
SG-MCMC can be computationally expensive. There are also approximate Bayesian inference methods introduced to accelerate the inference procedures~\citep{maddox2019simple,dusenberry2020}.
In particular, the Monte Carlo (MC) drop-out method sets some of the  network weights to zero according to a prior distribution~\citep{gal2017concrete}.
This method serves as a simple alternative to  variational Bayes methods which approximate the posterior~\citep{blundell2015weight,graves2011practical,louizos2016structured,rezende2014stochastic,qiu2019quantifying}.
We use the popular MC dropout \cite{gal2017concrete} method in the experiments. 


\subsection{A Recipe for UQ in Practice}
The above-mentioned UQ methods have different properties. Table~\ref{tb:comparison-table} shows an overall comparison of different UQ methods for deep spatiotemporal forecasting. The relative computation budget is displayed in the ``Computation'' row. The numbers represent the numbers of parallel cores required to finish computation within the same amount of time.
%

\begin{table}[h]
\vspace{-3mm}
  \centering
  \resizebox{\linewidth}{!}{
\begin{tabular}{c|c|c|c|c|c|ccc}
    \toprule                
    Method & Bootstrap & Quantile & SQ & MIS  &  MC Dropout &  SG-MCMC \\
    \midrule
 Computation & 25 & 1 & 1 & 1 & 1 & 25 \\
 Small sample & \xmark & \xmark & \xmark & \xmark & \xmark & \checkmark \\
 Consistency & \checkmark & \xmark & \xmark & \xmark & \xmark & \checkmark \\
 Accuracy & \checkmark & \checkmark & \checkmark & \checkmark &  \checkmark & \checkmark\checkmark \\
 Uncertainty & \xmark & \checkmark & \xmark & \checkmark\checkmark & \xmark & \checkmark \\
    \bottomrule
  \end{tabular}
  }
    \caption{Comparison of different deep uncertainty quantification methods for forecasting. Double check marks represent robustly highest performance in experiments. 
 }
 \label{tb:comparison-table}
\vspace{-7mm}
\end{table}

 
It is worth noting that when comparing computational budgets, we assume only one level of confidence (usually $95\%$) is estimated.
If $n$ confidence levels are required, then the computation complexity of quantile regression and MIS regression multiplies by $n$, whereas the computational budgets for the rest of the methods remain unchanged.
Through our experiments, we provide the following recipe for practitioners.

\begin{itemize}
\item \textbf{Large data, sufficient computation}: We recommend SG-MCMC and bootstrap as both methods generate accurate predictions, high-quality uncertainty quantification, which means the corresponding MIS is small, and have asymptotic consistency.
\item \textbf{Large data,  limited computation}: We recommend MIS  and quantile regression. Both of them prevail in providing accurate results with high-quality uncertainty quantification.  

%
\item \textbf{Small data}: Bayesian learning with SG-MCMC can have an advantage here. By choosing a proper prior, SG-MCMC can lead  to better generalization. Frequentist methods are often inferior with very limited samples, especially for mean prediction, see  experiments for more details. 
\item \textbf{Asymptotic consistency}: Both bootstrap and SG-MCMC methods are asymptotically consistent, making them the default choice when the computational budget and the dataset are sufficient.
When the computational budget is restrictive, the comparison is about sample complexity: the number of bootstrap resampling versus the number of MCMC chains determines which method needs more parallel computing resources.
We found in our experiments that the sample complexity of bootstrap is consistently higher than that of posterior sampling.
\end{itemize}


\section{Experiments}
We benchmark the performance of both frequentist and Bayesian UQ methods on many spatiotemporal forecasting applications, including air quality PM2.5, road network traffic, and COVID-19 incident deaths. Air quality PM2.5  is on a regular grid, while traffic and COVID-19 are on  graphs.

\out{The experiments are implemented using pytorch \citep{paszke2019pytorch}.
Based on our observations, Bayesian methods typically reach lower errors in their mean predictions while frequentist methods, especially quantile  and MIS regression, prevail in estimating the confidence levels.
}

\subsection{Datasets}

\paragraph{Air Quality PM 2.5}
The air quality PM 2.5 dataset contains hourly weather and air quality data in Beijing provided by the KDD competition of 2018 \footnote{\url{https://www.biendata.xyz/competition/kdd_2018/}}. The weather data  is at the  grid level ($21\times31$), which contains temperature, pressure, humidity, wind direction, and wind speed,  forming a $21\times31\times5$ tensor at each timestamp. The air quality data is from 35 stations in Beijing. The air quality index we considered is PM2.5 measured in micrograms per cubic meter of air ($\mu g/m^{3}$). According to the air quality index provided by the U.S. Environmental Protection Agency \citep{us2012revised}, the air quality is unhealthy from 101 to 300, and hazardous from 301 to 500.  We follow the setting in \citep{liu2018air} to map PM2.5 to grid-based data using the weighted average over stations:

\begin{equation}
    PM2.5_{k} = \frac{1}{\sum_{i}{\frac{1}{d_{i,k}^{2}+\epsilon}PM2.5_{i}}}
    \label{eqn:interpolation}  
\end{equation}

Here $d_{1,cell}$ is the distance between station i and the grid cell k. $\epsilon$ is introduced to prevent divide-by-zero error.

\paragraph{METR-LA Road Network Traffic}
The {METR-LA} \citep{jagadish2014big} dataset contains $4$ months vehicle speed information from loop detector sensors in Los Angeles County highway system. The maximum speed limit on most California highways is 65 miles per hour. The task is to forecast the traffic speed for $207$ sensors  simultaneously. We  build the spatial correlation matrix $A$ using  road network distances between sensors $v_i$ and $v_j$ and  a Gaussian kernel:
\[
    A_{ij} = \exp\Big(- \frac{d(v_i,v_j)}{\sigma^2}\Big)
\]

\paragraph{COVID-19 Incident Death}
The COVID-19  dataset  contains reported deaths from Johns Hopkins University \citep{dong2020interactive} and the death predictions from a mechanistic, stochastic, and spatial metapopulation epidemic model called Global Epidemic and Mobility Model (GLEAM) \citep{balcan2009multiscale, balcan2010modeling,tizzoni2012real,zhang2017spread,chinazzi2020effect,davis2020estimating}. Both data are recorded for the $50$ US states during the period from May 24th to Sep 12th, 2020. We use the residual between the reported death and the corresponding GLEAM predictions to train the model (\url{http://covid19.gleamproject.org/}).

{We construct the spatial correlation matrix $A$ using air traffic between different states. Each element in the matrix is the average number of passengers traveling between two states daily.
The air traffic data were obtained from the Official Aviation Guide (OAG) and the International Air Transportation Association (IATA) databases (updated in $2021$) \cite{OAG,IATA}. See Appendix~\ref{app:covid_deepgleam} for details.}
\out{This is highly challenging as the data is very small, highly noisy, and pertains to complex spatial dependency. We chose this dataset as its unique challenges make it an ideal task to test and compare different UQ methods for spatiotemporal forecasting.  }

\subsection{Experiment Setup}
For all experiments, we start with a deep learning  point estimate model (Point) based on Sec. \ref{sec:deep_models}. We use as deep learning models ConvLSTM \cite{xingjian2015convolutional} for grid-based data and DCRNN \cite{li2018diffusion} for graph-based data. Then we modify and adapt the  deep learning models to incorporate 6 different UQ methods as in Sec.\ref{sec:uq_framework}.
\out{including:
\begin{itemize}
    \item Bootstrap (F)
    \item Quantile Regression (F)
    \item Spline Quantile Regression (F)
    \item MIS Regression (F)
    \item Monte Carlo dropout (B)
    \item SG-MCMC (B)
\end{itemize}
where (F) indicates frequentist methods and (B) standards for Bayesian methods.   
}

For the ConvLSTM model \cite{xingjian2015convolutional}, we follow the best setting  reported in \citep{liu2018air}  with $3\times3$ filter, 256 hidden states, and 2 layers. We applied early stopping, autoregressive learning, and  gradient clipping \citep{zhang2019gradient} to improve the generalization. As PM2.5 is closely related to weather, we use  a combined loss
\begin{equation*}
    \mathcal{L} = MAE(weather) \times 0.2 + MAE(PM2.5)
    \label{eqn:loss_fun}  
\end{equation*}
%

For the DCRNN model, we use the exact setting  in \citep{li2018diffusion} for the traffic dataset and use calibrated network distance to construct the graph.  We applied early stopping,  curriculum learning, and  gradient clipping \citep{zhang2019gradient} to improve the generalization.  For the COVID-19 dataset, we chose to forecast the residual rather than directly predicting the death number because the latter has lower accuracy due to limited samples, see Sec. \ref{app:dcrnn_vs_deepgleam} for quantitative comparisons.

\out{Instead, we use the difference between daily death number and GLEAM predictions as input and output of the DCRNN model and name it DeepGLEAM model.  We subtract the learned difference from the GLEAM outputs to improve its prediction. Essentially, we use DCRNN (deep neural networks) to learn the correction terms for the mechanistic model GLEAM.}

We report the mean absolute error (MAE) as a point estimate metric. For UQ metric, we use MIS suggested in Sec \ref{sec:MIS} for 95\% confidence intervals. As MIS combines both the interval and coverage, we also report the interval score, which is the average difference between the upper and lower bounds.

\begin{table*}[t!]
\centering
\resizebox{0.76\textwidth}{!}{
\scriptsize
\begin{tabular}{c|c|ccccccc}
\toprule
Horizon \textit{H} & Metric & Point & Bootstrap (F) & Quantile (F) & SQ (F) & MIS(F) & MC Dropout (B)& SG-MCMC  (B) \\ 
\midrule
\multicolumn{9}{c}{\textbf{Air Quality PM 2.5}}\\
\midrule
\multirow{3}{*}{first 12 hrs} & \edits{MAE} & 22.93 & 22.27 & 24.89 & 23.35 & 24.41 & 22.11 & \textbf{21.83}  \\ 
& MIS & --- & 469.36 & 157.53 & 188.30 & \textbf{147.68} & 736.50 & 223.63    \\ 
& \edits{Interval width} & --- & 28.80 & 115.25 & 82.45 & 137.34 & 8.47 & 57.44 \\ 
\midrule
\multirow{3}{*}{first 24 hrs} & \edits{MAE} & 25.17 &  23.37 &  26.25  &  24.67 &  25.93 &  24.42 &   \edits{\textbf{23.10}} \\ 
& MIS & --- & 515.51 & 171.12 & 217.34 & \textbf{156.02} & 828.42 & 279.21    \\ 
& \edits{Interval width} & --- & 35.35 & 121.51 & 84.39 & 141.74 & 8.46 & 56.22 \\ 
\midrule
\multirow{3}{*}{total 48 hrs} & \edits{MAE} & 26.77 &  25.69 &  27.82  &  26.47 &  30.20 &  26.03 &   \edits{\textbf{24.73}} \\ 
& MIS & --- & 531.19 & 189.35  & 253.15 & \textbf{179.96} & 881.05 & 315.34    \\ 
& \edits{Interval width} & --- & 31.43 & 130.17 & 90.09 & 151.61 & 9.16 & 55.66  \\ 
\midrule
\multicolumn{9}{c}{\textbf{METR-LA Traffic}}\\
\midrule
\multirow{3}{*}{15 min} & \edits{MAE} & \edits{2.38} &  \edits{2.63} &  \edits{2.43}  &  \edits{2.67} &  \edits{2.63} &  \edits{2.47} &   \edits{\textbf{2.32}} \\ 
& MIS & --- & 39.76 & 18.32  & 29.04 & \textbf{18.26} & 27.61 & 32.21    \\ 
& \edits{Interval width} & --- & \edits{5.80} & \edits{12.42} &\edits{8.38} & \edits{12.46} & \edits{8.99} & \edits{8.73}  \\ 
%
\midrule
\multirow{3}{*}{30 min} & \edits{MAE} & \edits{2.73} &  \edits{3.19} &  \edits{2.79}  &  \edits{3.10} &  \edits{3.08} &  \edits{2.94} &   \edits{\textbf{2.54}} \\ 
& MIS & --- & 38.48  &  21.54  & 40.93 & \textbf{21.09} & 33.38 & 31.87  \\ 
& \edits{Interval width} & --- & \edits{7.86} & \edits{13.48} & \edits{8.37} & \edits{13.99} & \edits{11.10} & \edits{12.62}  \\ 

%
\midrule
\multirow{3}{*}{1 hour} & \edits{MAE} & \edits{3.14} &  \edits{3.99} &  \edits{3.19}  &  \edits{3.75} &  \edits{3.65} &  \edits{3.70} &   \edits{\textbf{3.00}} \\ 
& MIS & --- & 38.58 &  25.74 & 60.56 & \textbf{24.33} & 43.11 & 30.35   \\ 
& \edits{Interval width} & --- & \edits{11.65} & \edits{14.50} & \edits{8.38} & \edits{15.55} & \edits{14.46} & \edits{18.79}  \\ 
\midrule
\multicolumn{9}{c}{\textbf{COVID-19 Incident Death}}\\
\midrule
\multirow{3}{*}{1W} & {MAE} & {34.32} & {32.57} & {36.63} & {34.42} & {39.03} & {34.27} & {\textbf{29.72}}  \\ 
& MIS  & --- & 856.87 & \textbf{413.37} & 1049.69 & 427.53 & 790.24 & 563.77\\ 
& {Interval width} & --- & {32.48} & {190.98} & {23.73} & {227.19} & {47.79} & {47.05}\\ 
\midrule
\multirow{3}{*}{2W} & \edits{MAE} & \edits{33.72} & \edits{32.38} & \edits{36.81} & \edits{34.03} & \edits{36.35} & \edits{33.64} & \edits{\textbf{27.65}}  \\ 
& MIS  & --- & 762.55 & \textbf{363.14} & 1010.97 & 379.27 & 686.43 & 599.14\\ 
& \edits{Interval width} & --- & \edits{36.25} & \edits{219.06} & \edits{24.46} & \edits{260.24} & \edits{55.93} & \edits{45.45}\\ 
\midrule
\multirow{3}{*}{3W}& \edits{MAE} & \edits{41.37} & \edits{40.33} & \edits{44.29} & \edits{41.24} & \edits{43.15} & \edits{41.26} & \edits{\textbf{34.62}}  \\ 
& MIS & --- & 1028.63  & 411.05 & 1292.98 & \textbf{402.46} & 905.66 & 821.71\\ 
& \edits{Interval width} & --- & \edits{39.95} & \edits{242.47} & \edits{24.16} & \edits{291.96} & \edits{62.39} & \edits{46.16}\\ 
%
\midrule
\multirow{3}{*}{4W} & \edits{MAE} & \edits{42.37} & \edits{41.71} & \edits{46.20} & \edits{41.79} & \edits{44.45} & \edits{42.28} & \edits{\textbf{40.66}}  \\ 
& MIS & --- &  1035.26& 455.27 & 1303.02 & \textbf{428.82} & 891.45 & 852.26\\ 
& \edits{Interval width} & --- & \edits{43.61} & \edits{262.09} & \edits{23.85} & \edits{316.13} & \edits{67.52} & \edits{47.58}\\ 
%
\bottomrule
\end{tabular}}

\caption{Performance comparison of 6 different UQ methods applied to grid-based air quality PM2.5, graph-based traffic and COVID-19 incident death forecasting. (F) indicates frequentist methods and (B) standards for Bayesian methods. We compare these methods using the mean absolute error (MAE), the mean interval score (MIS), and the width of the confidence bounds.}
\label{tab:accuracy}
\vspace{-7mm}
\end{table*}

\subsection{Implementation Details}
We adapt ConvLSTM and DCRNN to implement various UQ methods with the following parameter settings. 

\textbf{Bootstrap}
We randomly dropped 50$\%$ of training data while keeping the original validation and testing data. We obtain 25 samples for constructing mean prediction and confidence intervals. 

\textbf{Quantile regression}
We apply the pinball loss function \citep{Koenker2005} with the corresponding quantile (0.025, 0.5, 0.975). The model and learning rate setup is the same as point estimation. 

\textbf{SQ regression}
We use the linear spline quantile function to approximate a quantile function and use the CRPS as the loss function. For every point prediction, there are 11 trained parameters to construct the quantile function. The 1\ts{st} parameter is the intercept term. The next 5 can be transformed to the slopes of 5 line segments. The last 5 can be transformed to a vector of the 5 knots' positions. The model and learning rate setup is the same as point estimation. 

\textbf{MIS regression}
We combine MAE with MIS (\eqref{eqn:MIS_loss}) and directly minimize this loss function. The model and learning rate setup is the same as \cite{li2018diffusion}. 

\textbf{MC Dropout}
We implement the algorithm provided by \citep{zhu2017deep} and simplify the model by only considering the model uncertainty. We apply random dropout through the testing process with 5\% drop rate and iterate 50 times to achieve a stable prediction. The result for comparison averages the performance of 10 trails. 

\textbf{SG-MCMC}
We set the learning rate of SG-MCMC $h=5e^{-4}$, as defined in Eqn. \ref{eq:SGNHT}, and we selected a Gaussian prior $\theta\sim\normal(0, 4.0)$ with random initialization as $\theta_0\sim\NORMAL(0, 0.2)$. We selected different priors and initialization parameters around the choices above. The performances are fairly stable for different choices as long as the model would not diverge (e.g., random initialization variance too large). We note here a symmetric Gamma prior could also work in this case with $\theta\sim\Gamma(0.1, 1)$. We apply early stopping at epoch 50, and it helps to improve the generalization on testing set. Our result is averaged from 25 posterior samples.

%
We use V100 GPUs to perform all the experiments. \edits{In addition to the MAE and MIS scores discussed in section~\ref{sec:MIS}, we also include the widths of the confidence bounds to demonstrate the coverage.
}

\out{
These interval widths demonstrate that bootstrap and SQ methods tend to be overconfident and provide narrower intervals, missing coverage over many data points, as compared with Quantile and MIS methods.
\edits{The interval in Table \ref{tab:covid_comparison} shows  bootstrap, SQ, MC dropout, and SG-MCMC methods tend to make overconfident predictions with smaller confidence interval bounds compared with Quantile, and MIS methods.}

Table \ref{tab:traffic} compares  UQ methods for 15 minutes, 30 minutes, and \ry{missing?}. The point estimate model is DCRNN which we use as a reference. We observe that Bayesian methods (SG-MCMC) lead to better prediction accuracy, outperforming the point estimate.
Meanwhile, frequentist methods generally outperform Bayesian methods in \edits{$95\%$} confidence intervals.

Directly minimizing MIS score function achieves the best performance in uncertainty. Its overall performance is similar to quantile regression.  MC dropout has a relatively good prediction accuracy with a reasonable MIS score, especially for the short-term forecast. The uncertainty estimation is robust within a reasonable range of dropout rate \citep{zhu2017deep}. We use 5\% dropout rate in this experiment. We note that bootstrap performs poorly for both metrics, most likely due to limited  samples.

SQ method is implemented to deal with the quantile crossing problem. 
We use 5 knots to build the spline quantile functions, which might not be enough for prediction. Both prediction accuracy and MIS are not good for this method. 

}




\out{
\begin{table*}[!h]
\centering
\resizebox{\textwidth}{!}{
\scriptsize
\begin{tabular}{c|c|ccccccc}
\toprule
$T$ & Metric  & Point & Bootstrap (F) & Quantile (F) & SQ (F) & \MIS(F) & MC Dropout (B)& SG-MCMC  (B) \\ 
\midrule
\multirow{3}{*}{15 min} & \edits{MAE} & \edits{2.38} &  \edits{2.63} &  \edits{2.43}  &  \edits{2.67} &  \edits{2.63} &  \edits{2.47} &   \edits{\textbf{2.32}} \\ 
& MIS (95$\%$ CI)& --- & 39.76 & 18.32  & 29.04 & \textbf{18.26} & 27.61 & 32.21    \\ 
& \edits{Interval} & --- & \edits{5.80} & \edits{12.42} &\edits{8.38} & \edits{12.46} & \edits{8.99} & \edits{8.73}  \\ 
%
\midrule
\multirow{3}{*}{30 min} & \edits{MAE} & \edits{2.73} &  \edits{3.19} &  \edits{2.79}  &  \edits{3.10} &  \edits{3.08} &  \edits{2.94} &   \edits{\textbf{2.54}} \\ 
& MIS (95$\%$ CI) & --- & 38.48  &  21.54  & 40.93 & \textbf{21.09} & 33.38 & 31.87  \\ 
& \edits{Interval} & --- & \edits{7.86} & \edits{13.48} & \edits{8.37} & \edits{13.99} & \edits{11.10} & \edits{12.62}  \\ 

%
\midrule
\multirow{3}{*}{1 hour} & \edits{MAE} & \edits{3.14} &  \edits{3.99} &  \edits{3.19}  &  \edits{3.75} &  \edits{3.65} &  \edits{3.70} &   \edits{\textbf{3.00}} \\ 
& MIS (95$\%$ CI)& --- & 38.58 &  25.74 & 60.56 & \textbf{24.33} & 43.11 & 30.35   \\ 
& \edits{Interval} & --- & \edits{11.65} & \edits{14.50} & \edits{8.38} & \edits{15.55} & \edits{14.46} & \edits{18.79}  \\ 
%
\bottomrule
\end{tabular}}
\caption{Performance comparison of 6 different UQ methods applied to graph-based  \emph{METR-LA} traffic  forecasting. }
\label{tab:traffic}
\vspace{-5mm}
\end{table*}

We use $70$-$10$-$20$ split of data for training, validation and  testing. Missing values are excluded. For all training, validation, and testing datasets, we are using an hour's traffic data to predict the next 5, 10, 15, ..., 60 minutes of traffic. 
}


\out{
We apply 6 UQ methods to DCRNN:  {Bootstrap (with $25$ resampled datasets)}, {quantile} regression, Spline quantile regression (SQ), and MIS regression are frequentist methods. Monte Carlo Dropout,  SG-MCMC (with $25$ posterior samples) are Bayesian methods. For MIS regression, we combine the Mean Absolute Error (MAE) score together with the Mean Interval Score (MIS) as the loss function. For SG-MCMC, we apply Stochastic Nos\'{e}-Hoover thermostat  \citep{ding2014bayesian} for sampling posterior.  See  Appendix \ref{Appendix:UQ_setup} for additional details.  
We report  Root Mean Square Error (RMSE) {and Mean Absolute Error (MAE)} for mean prediction  and Mean Interval Score (MIS) to assess the quality of the prediction uncertainty. 
}

\out{
\begin{table*}[h!]
\centering
\resizebox{\textwidth}{!}{
\scriptsize
\begin{tabular}{c|c|ccccccc}
\toprule
$T$ & Metric  & Point & Bootstrap (F) & Quantile (F) & SQ (F) & \MIS(F) & MC Dropout (B)& SG-MCMC  (B) \\  
\midrule
\multirow{3}{*}{1W} & {MAE} & {34.32} & {32.57} & {36.63} & {34.42} & {39.03} & {34.27} & {\textbf{29.72}}  \\ 
& MIS (95$\%$ CI) & --- & 856.87 & \textbf{413.37} & 1049.69 & 427.53 & 790.24 & 563.77\\ 
& {Interval} & --- & {32.48} & {190.98} & {23.73} & {227.19} & {47.79} & {47.05}\\ 
\midrule
\multirow{3}{*}{2W} & \edits{MAE} & \edits{33.72} & \edits{32.38} & \edits{36.81} & \edits{34.03} & \edits{36.35} & \edits{33.64} & \edits{\textbf{27.65}}  \\ 
& MIS (95$\%$ CI) & --- & 762.55 & \textbf{363.14} & 1010.97 & 379.27 & 686.43 & 599.14\\ 
& \edits{Interval} & --- & \edits{36.25} & \edits{219.06} & \edits{24.46} & \edits{260.24} & \edits{55.93} & \edits{45.45}\\ 
\midrule
\multirow{3}{*}{3W}& \edits{MAE} & \edits{41.37} & \edits{40.33} & \edits{44.29} & \edits{41.24} & \edits{43.15} & \edits{41.26} & \edits{\textbf{34.62}}  \\ 
& MIS (95$\%$ CI)& --- & 1028.63  & 411.05 & 1292.98 & \textbf{402.46} & 905.66 & 821.71\\ 
& \edits{Interval} & --- & \edits{39.95} & \edits{242.47} & \edits{24.16} & \edits{291.96} & \edits{62.39} & \edits{46.16}\\ 
%
\midrule
\multirow{3}{*}{4W} & \edits{MAE} & \edits{42.37} & \edits{41.71} & \edits{46.20} & \edits{41.79} & \edits{44.45} & \edits{42.28} & \edits{\textbf{40.66}}  \\ 
& MIS (95$\%$ CI)& --- &  1035.26& 455.27 & 1303.02 & \textbf{428.82} & 891.45 & 852.26\\ 
& \edits{Interval} & --- & \edits{43.61} & \edits{262.09} & \edits{23.85} & \edits{316.13} & \edits{67.52} & \edits{47.58}\\ 
%
\bottomrule
\end{tabular}}
\caption{Performance comparison of 6 different UQ methods for graph-based COVID-19 mortality forecasting.  }
\vspace{-5mm}
\label{tab:covid_comparison}
\end{table*}
}



\begin{figure*}[t!]
    \centering
    \begin{minipage}{.485\textwidth}
        \centering
        \includegraphics[width=0.7\textwidth,trim={20 60 20 0}]{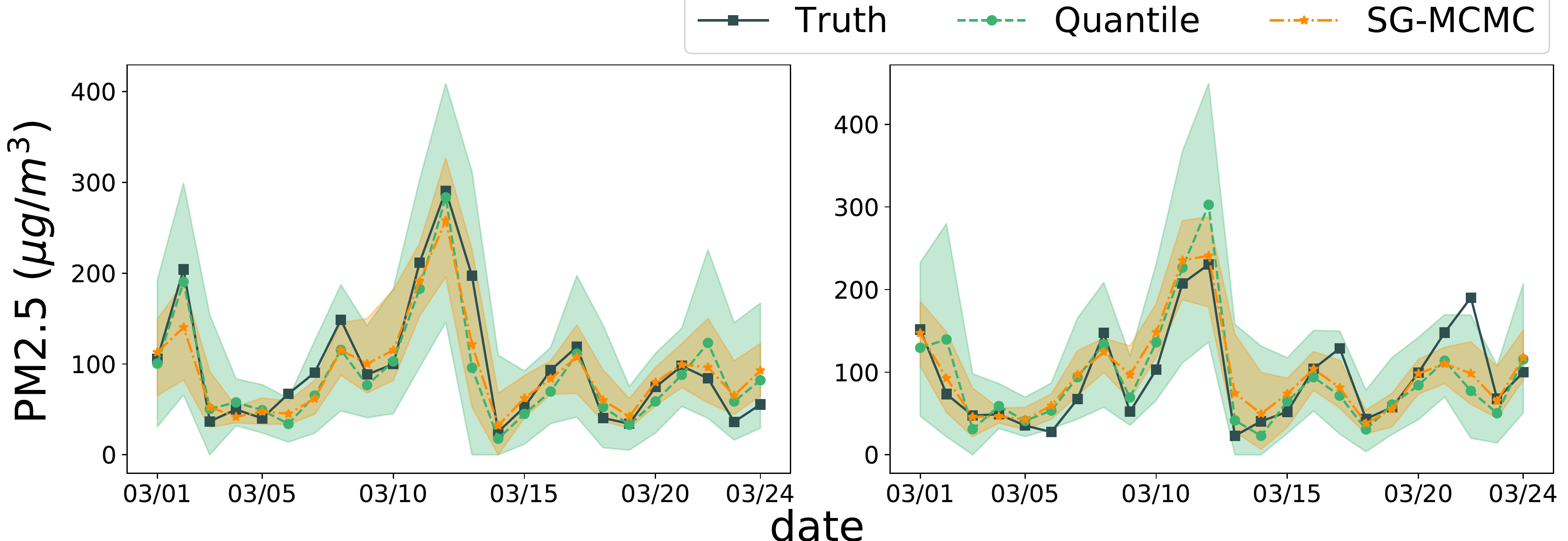}
        \caption{{12 hours ahead (Left) and 24 hours ahead (Right) Air quality PM 2.5 forecasting over 24 day time span on a randomly selected grid cell.}}
        \label{fig:pm25}
    \end{minipage}%
    \hfill
    \begin{minipage}{0.485\textwidth}
        \centering
        \includegraphics[width=0.7\textwidth,trim={20 60 20 0}]{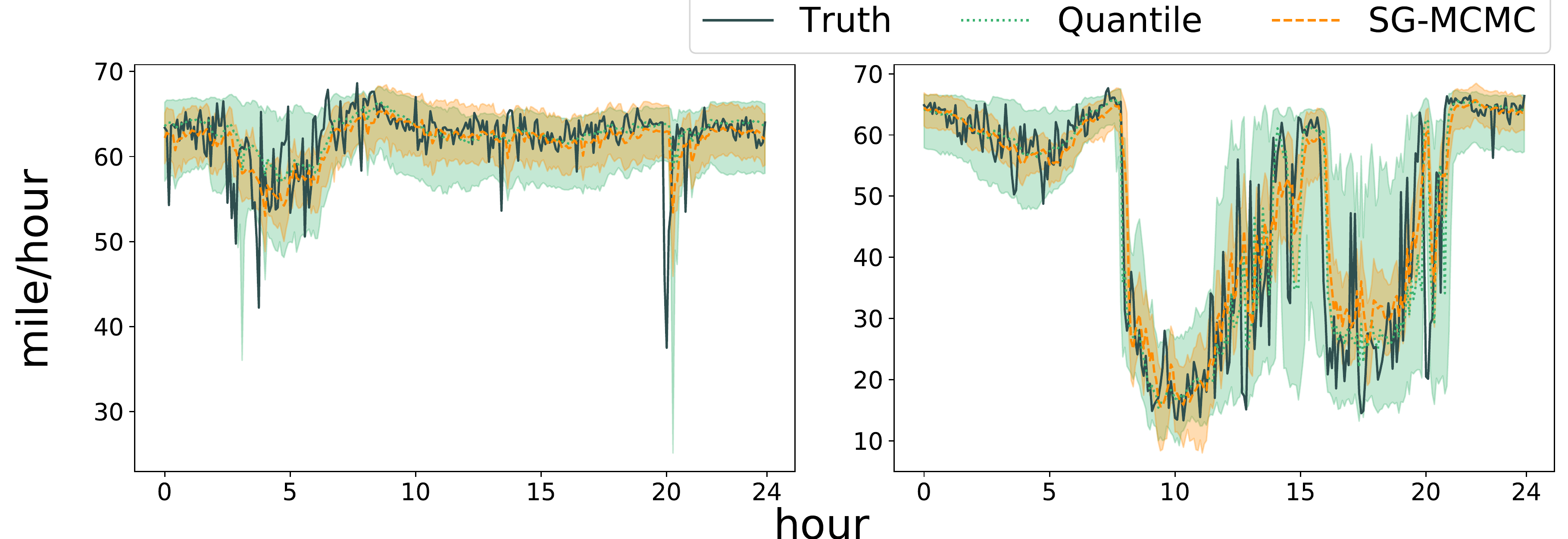}
        \caption{15 min ahead traffic forecasting over the 24 hour time span on two selected sensors. Left: regular hour traffic. Right: rush hour traffic.}
        \label{fig:traffic_qr_sgmcmc}
    \end{minipage}
    \begin{minipage}{1.0\linewidth}
        \centering
        \includegraphics[width=0.7\linewidth,trim={20 60 20 0}]{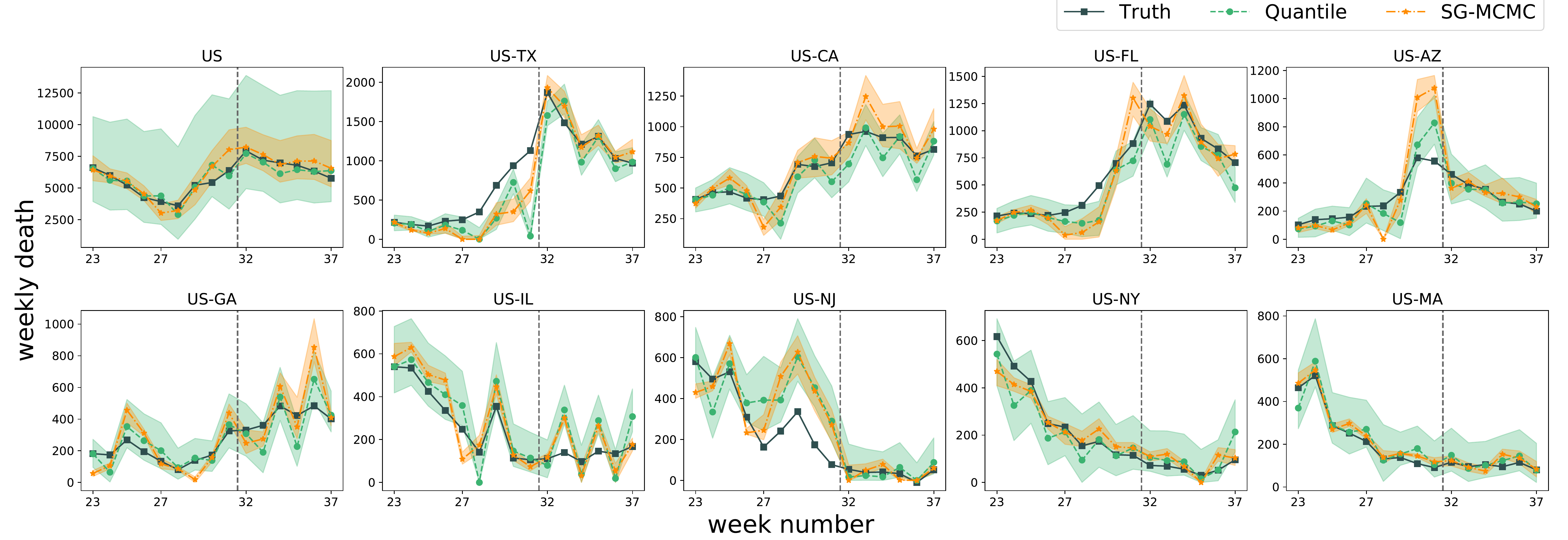}
        \caption{One week ahead COVID-19 incident death forecasting at the country level and 9 states in the United States with the heaviest death tolls from week 23 to week 37.}
        \label{fig:covid}
    \end{minipage}
    \vspace{-5mm}
\end{figure*}



\vspace{-3mm}
\subsection{Forecasting Performances}
Table \ref{tab:accuracy} compares the forecasting performances for 6 UQ methods across three different applications: (1)  PM 2.5 forecasting for next 12 - 48 hours on the hourly recorded weather and air quality dataset,  (2) 15min - 1 hour ahead traffic speed forecasting on {METR-LA}   and  (3) four weeks ahead mortality forecasting on COVID-19. Performances for air quality are aggregated for the entire sequence following the KDD competition standard. Others are reported with the corresponding timestamp.

Across the applications, we observe Bayesian method (SG-MCMC) leads to better prediction accuracy in MAE, even outperforming the point estimation models.
On the other hand, frequentist methods generally outperform Bayesian methods for the inference of \edits{ $95\%$} confidence intervals in MIS. 

In specific, directly minimizing MIS score function achieves the best performance in MIS, which is expected.  Quantile regression has a similar performance to the $\MIS$-regression, but under-performs in both MAE and MIS.  SQ enjoys good prediction accuracy, but has a larger MIS score compared to $\MIS$-regression. MC dropout has relatively good MAE, but the MIS scores are often the highest among all UQ methods. 


We note that bootstrap still under-performs most methods due to the limited amount of samples we have. SQ regression helps avoid the confidence interval crossing issue. However, it suffers from a large MIS due to the small number of knots for linear spline quantile function construction. For MC dropout, its prediction accuracy is relatively robust but its MIS is unsatisfactory.

\subsection{Forecasts Visualization}
We visualize the forecasts of Quantile regression (F) and SG-MCMC (B) for three different tasks in Figure~\ref{fig:pm25}, ~\ref{fig:traffic_qr_sgmcmc}, ~\ref{fig:covid}. Black solid lines indicate the ground truth. Green lines are the quantile regression predictions and orange lines are the forecasts from SG-MCMC. The shades represent the estimated $95\%$ confidence intervals. In general, we can see that Bayesian method SG-MCMC generates mean predictions much closer to the ground truth but often comes with narrower credible intervals. For example, in \textit{PM 2.5} prediction, the confidence intervals of SG-MCMC fail to cover the ground truth for both 24 and 48 hours prediction on 03/06 to 03/08. 

On the other hand, frequentist method quantile regression typically captures the true values with wider confidence interval bounds. This corresponds to its low MIS score reported in Table \ref{tab:accuracy}. Among all the tasks, the confidence interval bounds of quantile regression generally capture ground truth with $95\%$ probability. For COVID-19 forecasting, quantile regression suffers from limited amount of samples, occasionally resulting in poor performances in states including Texas (US-TX) and New Jersey (US-NJ). 

\section{Ablation study}

\subsection{Sample Complexity: Bootstrap and SG-MCMC}
Both bootstrap and SG-MCMC provide asymptotic consistency. Yet the sample complexities are often large for complex datasets. It is clear from Figure \ref{fig:sample_complexity} that the MIS score decreases as more Monte Carlo samples are drawn for both the bootstrap method and SG-MCMC posterior sampling method, indicating that sample variance is a major contribution to the inaccuracies. 
More importantly, we observe that the MIS score decreases faster with SG-MCMC than with bootstrap, signaling a smaller sample complexity for SG-MCMC.
\begin{figure}[ht]
    \centering
    \includegraphics[scale=0.35,trim={20 40 20 20}]{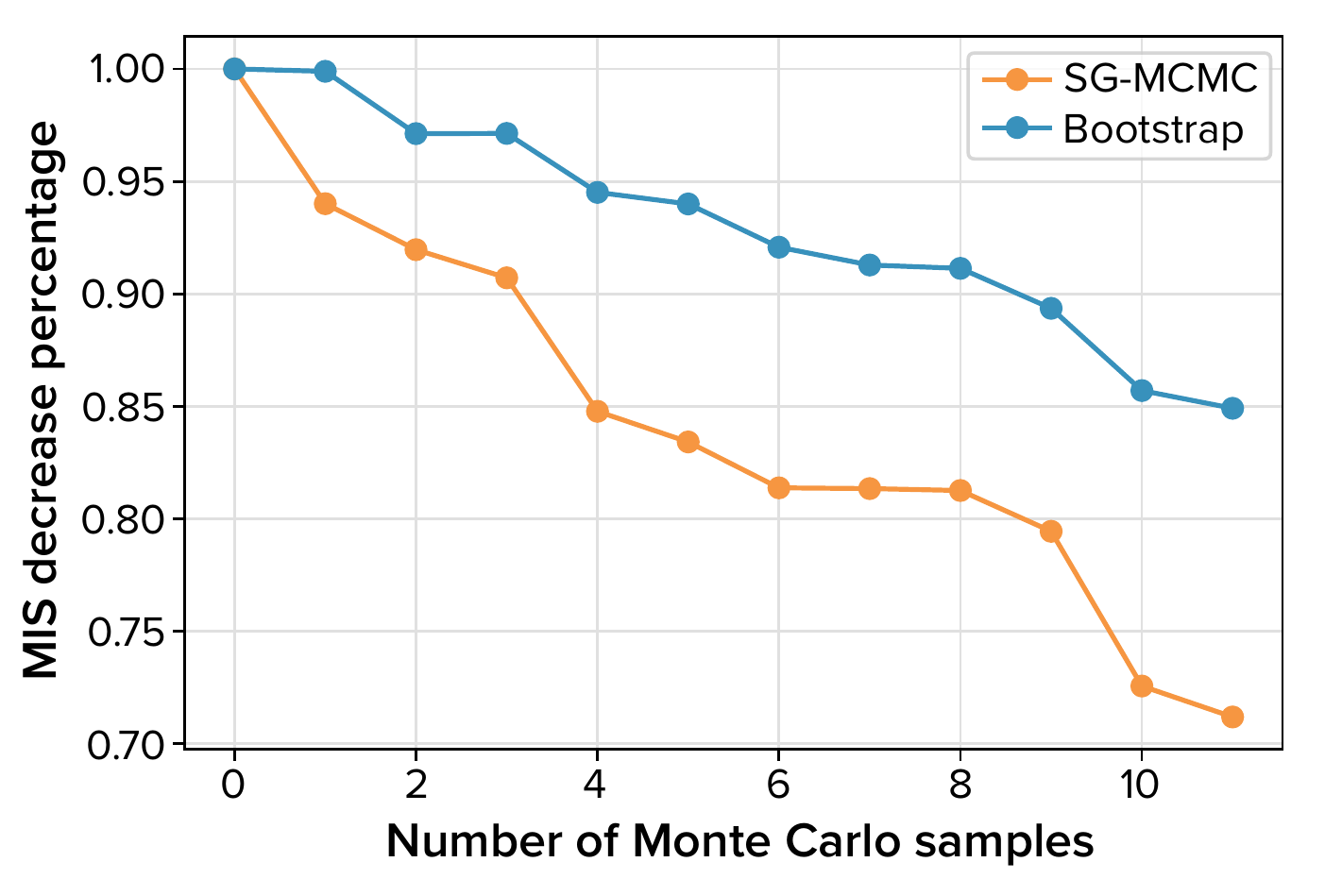}
    \caption{MIS of Bootstrap and SG-MCMC with increasing number of samples for COVID forecasting.}
    \label{fig:sample_complexity} 
    \vspace{-5mm}
\end{figure}
\subsection{COVID-19 models: mechanistic, deep learning and hybrid}
\label{app:dcrnn_vs_deepgleam}



For COVID-19 mortality forecasting, we have limited amount of data, which is insufficient to train a deep model. Instead, we train our model on the residual between the reported death and the corresponding GLEAM predictions, leading to a hybrid method called DeepGLEAM. Essentially, DeepGLEAM is learning the correction term in the mechanistic GLEAM model. An alternative approach is to directly train our model on the raw incident death counts, which we refer as Deep. 

Table \ref{tab:deep} shows the RMSE comparison among these models for four weeks ahead forecasting. We can see that DeepGLEAM  outperform the mechanics GLEAM model while the Deep model is much worse w.r.t point estimates. In fact, there is a $6.6\%$ improvement for DeepGLEAM and $17\%$ improvement for DeepGLEAM on average from GLEAM. 
\begin{table}[h]
\centering
\resizebox{0.7\columnwidth}{!}{
\begin{tabular}{c|ccc}
\toprule
Horizon $H$ & \edits{DeepGLEAM }  & \edits{Deep} & \edits{GLEAM}\\ \hline
{1W} & \edits{66.03} & \edits{239.94} & \edits{73.59}\\
\midrule
{2W}& \edits{57.67}& \edits{213.31} & \edits{65.46}\\
\midrule
{3W}& \edits{70.12}  & \edits{189.49} & \edits{73.75}\\
\midrule
{4W}& \edits{70.75} & \edits{161.88} & \edits{70.16}\\
\bottomrule
\end{tabular}}
\caption{RMSE comparison of different approaches for COVID-19 mortality forecasting}
\label{tab:deep}
\vspace{-8mm}
\end{table}

In Figure \ref{fig:dcrnn_vs_deepgleam}, we visualize the predictions from these three different methods: \edits{Deep}, GLEAM, and DeepGLEAM models in California. The input data (for training or validation) are from weeks before $34$, and we start the prediction from week $34$, labeled by the vertical dash line. It can be observed that the \edits{Deep model} fails to predict the dynamics of COVID-19 evolution. 
\begin{figure}[ht]
    \centering
    \includegraphics[scale=0.2,trim={20 60 20 30}]{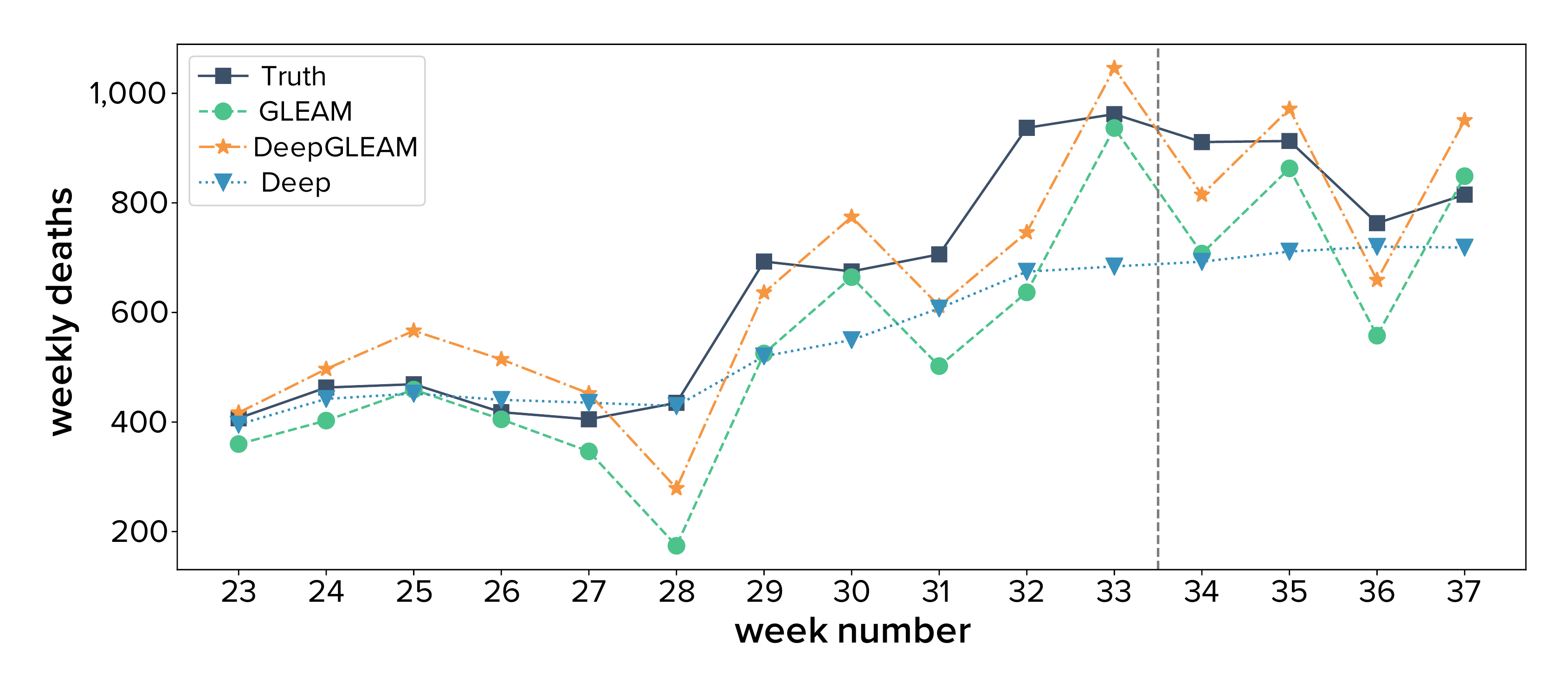}
    \caption{One week ahead COVID-19 forecasts visualization in California. Comparison shown for GEAM, DeepGLEAM and Deep models.}
    \label{fig:dcrnn_vs_deepgleam}
\vspace{-4mm}
\end{figure}
One reason for this phenomenon is the distribution shifts in the epidemiology dynamics.
With limited data, deep neural networks do not have enough information to guide their predictions.
On the other hand, GLEAM and DeepGLEAM leverage mechanistic knowledge about disease transmission dynamics  as inductive bias to infer the underlying dynamics change. 

\subsection{Anatomy of the performance: Coverage versus accuracy}
We conclude from Table~\ref{tab:accuracy} that in all the experiments we perform, better performance in MIS is related to not only higher accuracy but also larger confidence bounds.
This phenomenon indicates that the deep learning models we consider in this paper generally tend to be overconfident and benefit greatly from better capturing of all sources of variations in the data.

\section{Conclusion}
We conduct benchmark studies on  uncertainty quantification in deep spatiotemporal forecasting from both Bayesian and frequentist perspectives. Through experiments on air quality,  traffic and epidemics datasets, we conclude that, with a limited amount of computation, Bayesian methods are typically more robust in mean predictions, while frequentist methods are more effective in  covering variations in the data.
It is of interest to understand how to best combine Bayesian credible intervals with frequentist confidence intervals to excel in both mean predictions and confidence bounds.
Another future direction worth exploring is how to explicitly make use of the spatiotemporal structure of the data in the inference procedures.
Current methods we analyze treats minibatches of data as i.i.d. samples.
For longer time series with a graph structure, it would be interesting to leverage the spatiotemporal nature of the data to construe efficient inference algorithms~\citep{ma2017HMM,aicher2019,alaa2020frequentist}. One limitation of current method is that the confidence interval does not reflect the prediction accuracy. We hypothesize this is due to the limited samples and complexity of the dynamics.

\section*{Acknowledgement}
This work was supported in part by AWS Machine Learning Research Award, Google Faculty Award, NSF Grant \#2037745, and the DARPA W31P4Q-21-C-0014. D.W. acknowledges support from the HDSI Ph.D. Fellowship. M.C. and A.V. acknowledge support from grants HHS/CDC 5U01IP0001137, HHS/CDC 6U01IP001137, and from Google Cloud Research Credits program. The findings and conclusions in this study are those of the authors and do not necessarily represent the official position of the funding agencies, the National Institutes of Health, or the U.S. Department of Health and Human Services. 
\bibliographystyle{ACM-Reference-Format}
\bibliography{ref}

\clearpage

\appendix

\section{Theoretical Analysis}
\label{Appendix:theory}

\begin{proof}[Proof of Proposition~\ref{thm:mis}]
Since the distribution $\mathbb{P}_Z$ of $Z$ has probability density $p_Z$,
\begin{align*}
\MIS_{\infty}(u,l;\rho) &= (u-l) \\
&+ \frac{2}{\rho} \lrp{ \int_u^\infty (z-u) p_Z(z) \rd z + \int_{-\infty}^l (l-z) p_Z(z) \rd z }.
\end{align*}
We demonstrate in the following that the minimum of $\MIS_{\infty}(u,l;\rho)$ is achieved when $u$ and $l$ define the $(1-\rho)$ confidence level.

For simplicity of exposition, we let $p_Z$ be symmetric around $0$ and let $l=-u$.
Then 
\[
\MIS_{\infty}(u;\rho) = 2u + \frac{4}{\rho} \lrp{ \int_u^\infty (z-u) p_Z(z) \rd z }.
\]
Setting
\[
0 = \frac{\rd}{\rd u} \MIS_{\infty}(u;\rho)
= 2 - \frac{4}{\rho} \lrp{ \int_u^\infty p_Z(z) \rd z },
\]
we reach the conclusion that the upper bound $u^*$ that achieves the minimum MIS satisfies: $\int_{u^*}^\infty p_Z(z) \rd z = \frac{\rho}{2}$.
Therefore, $[l^*,u^*]$ defines the $(1-\rho)$ confidence level.
\end{proof}


\begin{proof}[Proof of Proposition~\ref{thm:mis_finite}]

We think about minimizing the MIS score over lower and upper bounds $l$ and $u$, given samples $z_1,\cdots,z_N$ from the posterior distribution:
\begin{align}
(u,l) &= {\arg\min}_{u>l; u, l \in \real} \MIS_N(u,l) \nonumber\\
&= \underset{u>l; u, l \in \real}{\argmin} \Big\{ (u-l) + \frac{2}{\rho N}\sum_{i=1}^N ((z_i-u)\ind{z_i>u} \nonumber\\
&+ (l-z_i)\ind{z_i<l}) \Big\}. \nonumber
\end{align}

We prove in the following that the minimum of MIS score is achieved when we sort $\{z_1,\cdots,z_N\}$ in an increasing order and take $l = z_{\lceil \rho \cdot N / 2 \rceil}$ and $u = z_{N - \lfloor \rho \cdot N / 2 \rfloor}$, which define the quantile of the empirical distribution formed by the samples $\{z_1,\cdots,z_N\}$.

We first note that $\MIS_N(u,l)$ is a continuous function with respect to $u$ and $l$.
Since $\MIS_N(u,l)$ is piece-wise linear, we simply need to check that 
\begin{align*}
\MIS_N(z_j^{-},l) &= (z_j-l) \\
&+ \frac{2}{\rho N} \sum_{i\neq j} \lrp{(z_i-u)\ind{z_i>u} + (l-z_i)\ind{z_i<l}} \\
&= \MIS_N(z_j^{+},l).
\end{align*}
The result holds similarly for $l$.

In what follows we prove that $\MIS_N(u,l)$ is decreasing in $l$ whenever $l \leq z_{\lceil \rho \cdot N / 2 \rceil}$.
Whenever $l \geq z_{\lceil \rho \cdot N / 2 \rceil}$, $MIS(u,l)$ is increasing in $l$.
We can use similar reasoning to prove that the minimum is achieved when $u = z_{N - \lfloor \rho \cdot N / 2 \rfloor}$.

Consider derivative of $\MIS_N(u,l)$ over $l$
\begin{align}
    \frac{\partial \MIS_N(u, l)}{\partial l}=-1 + \frac{2}{\rho\cdot N}\sum_{i=1}^N \frac{\partial\lrp{(l-z_i)\ind{z_i<l}}}{\partial l}. \label{partial_l_ZeroCondition1}
\end{align}
Since
\begin{align}
    (l-z_i)\ind{z_i<l} = \begin{cases}
        l-z_i & \text{if $z_i < l$} \\
        0 & \text{if $z_i \geq l$} 
  \end{cases},
\end{align}
\begin{align}
    \frac{\partial\lrp{(l-z_i)\ind{z_i<l}}}{\partial l} = \begin{cases}
        1 & \text{if $z_i < l$} \\
        0 & \text{if $z_i \geq l$} 
  \end{cases}.
\end{align}
When $l \leq z_{\lceil \rho \cdot N / 2 \rceil}$, $\frac{2}{\rho\cdot N}\sum_{i=1}^N \frac{\partial\lrp{(l-z_i)\ind{z_i<l}}}{\partial l} \leq 1$. 
Hence $\MIS_N(u,l)$ is decreasing in $l$ whenever $l \leq z_{\lceil \rho \cdot N / 2 \rceil}$. 
Whenever $l \geq z_{\lceil \rho \cdot N / 2 \rceil}$, $\frac{2}{\rho\cdot N}\sum_{i=1}^N \frac{\partial\lrp{(l-z_i)\ind{z_i<l}}}{\partial l} \geq 1$, making $\MIS_N(u,l)$ increasing in $l$.
Therefore, $l = z_{\lceil \rho \cdot N / 2 \rceil}$ is the minimum of $\MIS_N(u,l)$. 

The above two facts complete the proof and conclude that 
\begin{align}
{\arg\min}_{u>l; u, l \in \real} \MIS_N(u,l)=(z_{N-\lfloor \rho \cdot N / 2 \rfloor}, z_{\lceil \rho \cdot N / 2 \rceil}),
\end{align}
which define the quantile of the empirical distribution formed by the samples $\{z_1,\cdots,z_N\}$.
\end{proof}

\section{Experiment Details}

\subsection{Point Estimation Model Implementation}
\label{app:model_detail}
For the ConvLSTM experiment on air quality dataset, we concatenate the interpolated PM2.5 at grid level with the weather data and use 24-hour sequence as input to forecast PM2.5 in the next 48 hours. The training data is from January 2nd 2017 to January 31st 2018. We split it into 72-hour sequences. For every sequence, the data is normalized for each channel (temperature, pressure, humidity, wind direction, wind speed, PM2.5). We use one month data from February 1st to February 28th in 2018 for validation and the following month data from March 1st to March 26th for testing, using the first 26 and 24 24-hour records respectively as the input to make the next 48 hour's PM2.5 forecast.

For the DCRNN model used in traffic dataset, we use  two hidden layers of RNN with 64 units. The filter type is a dual random walk and the diffusion step is 2. The learning rate of DCRNN is fixed at $1e^{-2}$ with Adam optimizer. We perform early stopping at 50 epochs. 

For the COVID-19 dataset, We use an encoder-decoder sequence to sequence learning framework in the DCRNN structure.
The encoder reads as input a $7 \times 50 \times 4$ tensor that comprises the daily residuals between the observed death number and the GLEAM forecasts for the $50$ US states over $4$ different prediction horizons.
It encodes the information in $7$ hidden layers. 
The decoder produces forecasts of the weekly residuals between the true death number and the GLEAM forecasts for each state for the following $4$ weeks.
We perform autoregressive weekly death predictions (from one week ahead to four weeks ahead). 
%
The DCRNN model only has one hidden layer of RNN with 8 units to overcome the overfitting problem. The filter type is Laplacian and the diffusion step is 1. The base learning rate of DCRNN is $1e^{-2}$ and decay to $1e^{-3}$ at epoch 13 with Adam optimizer. We have a strict early stopping policy to deal with the overfitting problem. The training stops as the validation error does not improve for three epochs after epoch 13.
\begin{figure*}[t!]
    \centering
    \includegraphics[width=\textwidth]{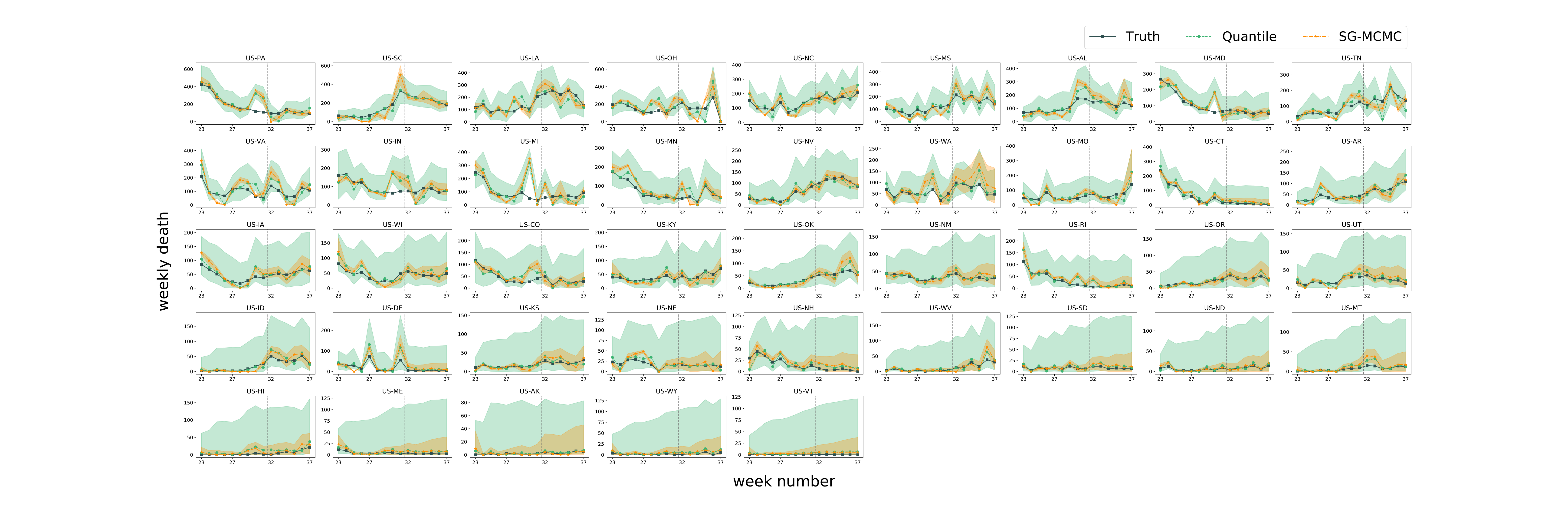}
    \caption{One week ahead COV-19 prediction for the rest of the  41 states in the United States. }
    \label{fig:covid_rest}
\end{figure*}
\begin{table*}[t!]
\centering
\caption{Performance comparison of different DeepGLEAM Models for COVID-19 mortality forecasting. }
\resizebox{0.9\textwidth}{!}{
\begin{tabular}{c|c|cccccccc}

\toprule
Horizon $H$ & Metric & Point & Bootstrap & Quantile & SQ & MAE-MIS & MC Drpoout & SG-MCMC \\
\midrule
\multirow{2}{*}{1W} & RMSE & 68.56 & 70.73 & 70.26 & 73.82 & 76.32 & 68.58 & \textbf{62.42}    \\ 
& MIS & --- & 1216.63 & 430.72 & 1299.78 & \textbf{424.49} & 831.86 & 715.40  \\ 
\midrule
\multirow{2}{*}{2W}& RMSE & 59.26 & 61.69  & 60.94 & 65.44 & 65.39 & 59.27 & \textbf{51.21}    \\ 
& MIS & --- & 1135.01 & \textbf{331.94} & 1272.82 & 368.20 & 796.54 & 727.27   \\ 
\midrule
\multirow{2}{*}{3W}& RMSE & 70.35 & 73.27 & 71.72 & 73.88 & 74.95 & 70.32 & \textbf{66.72}  \\ 
& MIS & --- & 1464.52 & \textbf{392.66} & 1559.62 & 416.23 & 997.65 & 930.55  \\ 
\midrule
\multirow{2}{*}{4W}& RMSE & 72.29 & 72.41 & 73.44 & 70.44 & 75.16 & 72.24 & \textbf{68.40}  \\ 
& MIS & --- & 1482.08 & \textbf{418.77} & 1586.95 & 468.49 & 1034.68 & 971.55 \\ 

\bottomrule

\end{tabular}}
\label{tb:ensemble_table}
\end{table*}

\subsection{Global Epidemic and Mobility Model.} \label{app:covid_deepgleam}
\out{
\begin{figure}
    \centering
    \includegraphics[scale=0.6]{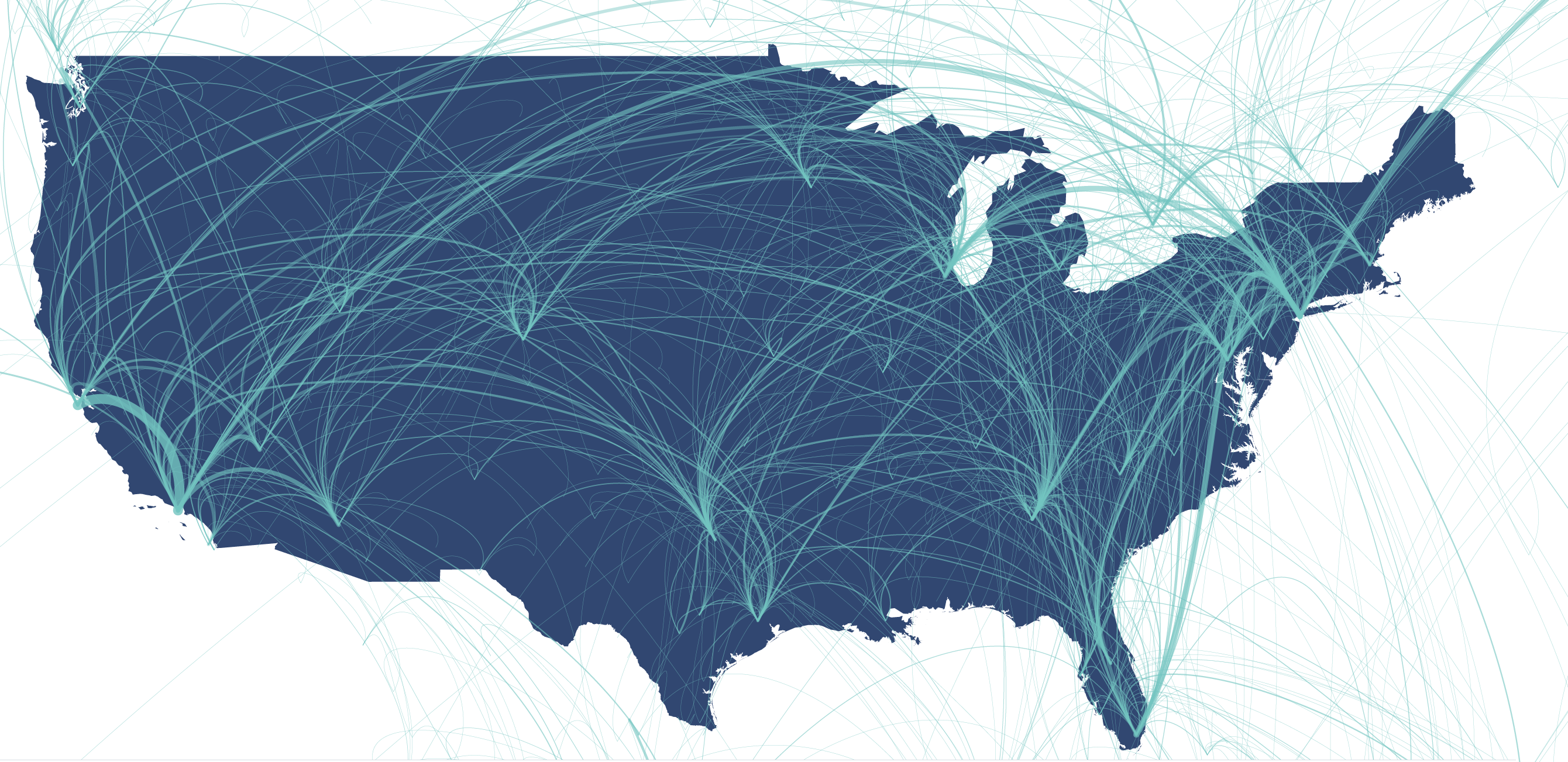}
    \caption{Original flight network connecting US airports. Data is aggregated at the State level to construct the network State-to-State graph.}
    \label{fig:flight}
 \vspace{-5mm}
\end{figure}
}
The Global Epidemic and Mobility model (GLEAM) is a stochastic, spatial, age-structured epidemic model based on a metapopulation approach which divides the world into over 3,200 geographic subpopulations constructed using a Voronoi tessellation of the Earth's surface. Subpopulations are centered around major transportation hubs (e.g. airports) and consist of cells with a resolution of approximately 25 x 25 kilometers \citep{balcan2009multiscale,balcan2010modeling,tizzoni2012real,zhang2017spread,chinazzi2020effect,davis2020estimating}. High resolution data are used to define the population of each cell. Other attributes of individual subpopulations, such as age-specific contact patterns, health infrastructure, etc., are added according to available data \citep{mistry2020inferring}.

GLEAM integrates a human mobility layer - represented as a network - that uses both short-range (i.e., commuting) and long-range (i.e., flight) mobility data from the Offices of Statistics for $30$ countries on $5$ continents as well as the Official Aviation Guide (OAG) and IATA databases (updated in $2021$). The air travel network consists of the daily passenger flows between airport pairs (origin and destination) worldwide mapped to the corresponding subpopulations. Where information is not available, the short-range mobility layer is generated synthetically by relying on the “gravity law”  or the more recent “radiation law”  both calibrated using real data \citep{Simini2012}. 

The model is calibrated to realistically describe the evolution of the COVID-19 pandemic as detailed in \cite{chinazzi2020effect,davis2020estimating}.
Lastly, GLEAM is stochastic and produces an ensemble of possible epidemic outcomes for each set of initial conditions. To account for the potentially different reporting levels of the states, a free parameter Infection Fatality Rate (IFR) multiplier is added to each model. To calibrate and select the most reasonable outcomes, we filter the models by the latest hospitalization trends and confirmed cases trends, and then we select and weight the filtered models using Akaike Information Criterion \citep{zhang2017forecasting}. The forecast of the evolution of the epidemic is formed by the final ensemble of the selected models.

\subsection{Additional Results}
Figure \ref{fig:covid_rest} visualizes the forecasts of Quantile regression (F) and SG-MCMC(B) for COVID-19 for the rest of the 41 states in U.S. Black solid lines indicate the ground truth. Green lines are the quantile regression predictions and orange lines are the forecasts from SG-MCMC. The shades represent the estimated 95\%confidence intervals. In general, we can see that SG-MCMC generates mean predictions much closer to the ground truth but often comes with narrower credible intervals

In addition to autoregressive forecasting, we also build an ensemble DeepGLEAM model by training based on the ensembled data along the forecasting horizon as the output.
We use the ensemble DeepGLEAM model which predicts the next 4 weeks death prediction together from the first decoder. Table \ref{tb:ensemble_table} shows the performance among the UQ methods, which shares similar results with the autoregressive model.

\end{document}